\documentclass{article}
\usepackage[letterpaper,top=1in,bottom=1in,left=1in,right=1in]{geometry}

\usepackage{xcolor}
\usepackage{xargs} 
\usepackage[textsize=footnotesize]{todonotes}
\newcommandx{\rc}[2][1=]{\todo[linecolor=green,
			backgroundcolor=blue!10,bordercolor=green,#1]{RC: #2}}
\newcommandx{\sh}[2][1=]{\todo[linecolor=green,
			backgroundcolor=red!10,bordercolor=green,#1]{HS: #2}}
\newcommandx{\jy}[2][1=]{\todo[linecolor=green,
			backgroundcolor=orange!10,bordercolor=green,#1]{JY: #2}}

\usepackage{tabularx}
\usepackage{amssymb,amsmath,amsthm}
\usepackage{epsfig}
\usepackage{textcomp}
\usepackage{epstopdf}
\usepackage{url}
\usepackage{cite}
\usepackage{wrapfig}
\usepackage{enumerate}
\usepackage{soul}
\usepackage{tipa}
\usepackage[percent]{overpic}
\usepackage{tikz}
\usepackage{pgfgantt}
\usepackage[linesnumbered,ruled,vlined]{algorithm2e}

\newtheorem{theorem}{Theorem}

\newtheorem{proposition}[theorem]{Proposition}

\newtheorem{lemma}[theorem]{Lemma}

\usepackage{mathtools,xparse}

\makeatletter
\newcommand{\customlabel}[2]{%
\protected@write \@auxout {}{\string \newlabel {#1}{{#2}{}}}}
\makeatother

\newif\ifdraft
\drafttrue

\newif\ifoc 
\octrue

\usepackage{xspace}

\def\isag{i\textsc{SaG}\xspace}

\def\paf{\textsc{PaF}\xspace}
\def\paft{\textsc{PaFT}\xspace}
\def\oldr{\textsc{OLDR}\xspace}
\def\oldrdisc{\textsc{DiscretizeOLDR}\xspace}
\def\triilp{\textsc{TriILP}\xspace}
\def\hexilp{\textsc{HexILP}\xspace}
\def\S{\mathcal S}
\def\G{\mathcal G}
\def\W{\mathcal W}

\title{Coordinating the Motion of Labeled Discs with Optimality Guarantees 
under Extreme Density}

\author{Rupesh Chinta and Shuai D. Han and Jingjin Yu
\thanks{R. Chinta is with the Department of Electrical and Computer 
Engineering, Rutgers University at New Brunswick, E-mail: 
rupesh.chinta@rutgers.edu. S. D. Han and J. Yu is with the Department of 
Computer Science, Rutgers University at New Brunswick. E-mails: 
\{shuai.han, jingjin.yu\}@cs.rutgers.edu.
}
}
\begin{document}
\maketitle
\begin{abstract}
We push the limit in planning collision-free motions for routing uniform 
labeled discs in two dimensions. First, from a theoretical perspective, we 
show that the constant-factor time-optimal routing of labeled discs can be 
achieved using a polynomial-time algorithm with robot density over $50\%$ in 
the limit (i.e., over half of the workspace may be occupied by the discs). 
Second, from a more practical standpoint, we provide a high performance 
algorithm that computes near-optimal (e.g., $1.x$) solutions under the same 
density setting. 
\end{abstract}

\section{Introduction}
The routing of rigid bodies (e.g., mobile robots) to desired destinations 
under dense settings (i.e., many rigid bodies in a confined workspace) is a 
challenging yet high utility task. On the side of computational complexity, 
when it comes to feasibility (i.e., finding collision-free paths for moving 
the bodies without considering path optimality), it is well known that 
coordinating the motion of translating rectangles is PSPACE-hard 
\cite{HopSchSha84} whereas planning for moving labeled discs of variable 
sizes is strongly NP-hard in simple polygons \cite{SpiYak84}. More recently, 
it is further established that PSPACE-hardness extends to the unlabeled case 
as well \cite{SolHal15}. Since computing an arbitrary solution is already 
difficult under these circumstances, finding optimal paths (e.g., minimizing 
the task completion time or the distances traveled by the bodies) are at least 
equally hard. Taking a closer look at proof constructions in 
\cite{HopSchSha84,SolHal15,SpiYak84}, one readily observes that the 
computational difficulty increases as the bodies are packed more tightly 
in the workspace. On the other hand, in many multi-robot applications, 
it is desirable to have the capacity to have many robots efficiently and
(near-)optimally navigate closely among each other, e.g., in automated 
warehouses \cite{WurDanMou08,enright2011optimization}. Provided that 
per-robot efficiency and safety are not compromised, having higher robot 
density directly results in space and energy\footnote{With higher robot 
density, a fixed number of robots can fit in a smaller workspace, reducing the 
distance traveled by the robots} savings, thus enhancing productivity. 

As a difficult but intriguing geometric problem, the optimal routing of rigid 
bodies has received much attention in many research fields, particularly 
robotics. While earlier research in the area tends to focus on the 
structural properties and complete (though not necessarily scalable) algorithmic 
strategies \cite{ErdLoz86,LavHut98b,GhrOkaLav05,PeaClaMcp08,SolHal12}, more 
recent studies have generally attempted to provide efficient and scalable 
algorithms with either provable optimality guarantees or impressive empirical 
results, or both. For the unlabeled case, a polynomial-time algorithm from 
\cite{TurMicKum14} computes trajectories for uniform discs which minimizes the 
maximal path length traveled by any disc. The completeness 
of the algorithm depends on some clearance assumptions between the discs and 
between a disc and the environment. In \cite{SolYu15}, a polynomial-time 
complete algorithm algorithm is proposed also for unlabeled discs that 
optimizes the total travel distance, with a more natural clearance assumption 
as compared to \cite{TurMicKum14}. The clearance assumption (among others, 
the distance between two unit discs is at least $4$) translates to a maximum 
density of about $23\%$, i.e., the discs may occupy at most $23\%$ of the 
available free space. For the labeled case, under similar clearance settings, 
an integer linear programming (ILP) based method is provided in \cite{YuRus14RR} 
for minimizing solution makespan. Though without polynomial running time 
guarantee, the algorithm is complete and appears to performs well in practice. 
Complete polynomial-time algorithms also exist that do not require any 
clearance in the start and goal configurations \cite{HanRodYu18IROS}. However, 
the supported density is actually lower in this case as the algorithm needs to 
expand the start and goal configurations so that the clearance conditions 
in \cite{YuRus14RR} is satisfied. 

In this work, we study the problem of optimally routing labeled uniform unit 
discs in a bounded continuous two dimensional workspace. As the main result,
we provide a complete, deterministic, and polynomial-time algorithm that 
allow up to more than half of the workspace to be occupied by the discs while 
simultaneously ensuring $O(1)$ (i.e., constant-factor) time optimality of the 
computed paths. We also provide a practical and fast algorithm for the same 
setting without the polynomial running time guarantee. More concretely, our 
study brings the following contributions: {\em (i)} We show that when the 
distance between the centers of any two labeled unit discs is more than 
$\frac{8}{3}$, the continuous problem can be transformed into a multi-robot 
routing problem on a triangular grid graph with minimal optimality loss. A 
separation of $\frac{8}{3}$ implies a maximum density of over $50\%$. {\em 
(ii)} We develop a low polynomial-time constant-factor time-optimal algorithm 
for routing discs on a triangular grid with the constraint that no two 
discs may travel on the same triangle concurrently. {\em (iii)} We develop 
a fast and novel integer linear programming (ILP) based algorithm that computes 
time-optimal routing plans for the triangular grid-based multi-robot routing 
problem. Combining {\em (i)} and {\em (ii)} yields the $O(1)$ time-optimal 
algorithm while combining {\em (i)} and {\em (iii)} results in the more 
practical and highly optimal algorithm. In addition, the $\frac{8}{3}$ 
separation proof employs both geometric arguments and computation-based 
verification, which may be of independent interest. 

Our work leans on graph-theoretic methods for multi-robot routing, 
e.g., \cite{KorMilSpi84,Yu18RSS,StaKor11,WagChoC11,boyarski2015icbs,cohen2016improved,
YuLav16TRO}. In particular, our constant-factor time-optimal routing algorithm 
for the triangular grids adapts from a powerful routing method for rectangular 
grid in \cite{Yu18RSS} that actually works for arbitrary dimensions. However, 
while the method from \cite{Yu18RSS} comes with strong theoretical guarantee 
and runs in low polynomial time, the produced paths are not ideal due to the 
large constant factor. This prompts us to also look at more practical algorithms
and we choose to build on the fast ILP-based method from \cite{YuLav16TRO}, 
which allows us to properly encode the additional constraints induced by the 
triangular grid, i.e., no two discs may simultaneously travel along any 
triangle. 

\textbf{Organization}. The rest of the paper is organized as follows. We 
provide a formal statement of the routing problem and its initial treatment in 
Section~\ref{section:problem}. In Section~\ref{section:discretize}, we
show how the problem may be transformed into a discrete one on 
a special triangular grid. Then, in Section~\ref{section:o1-opt-algorithm}
and Section~\ref{section:ilp-model}, we present a polynomial time algorithm 
with $O(1)$-optimality guarantee and a fast algorithm that computes highly 
optimal solutions, respectively. We conclude in Section~\ref{section:conclusion}.

\section{Preliminaries}\label{section:problem}
\subsection{Labeled Disc Routing: Problem Statement}
Let $\W$ denote a closed and bounded $w \times h$ rectangular region. For 
technical convenience, we assume $w = 4n_1 + 2$ and $h = \frac{4}{\sqrt{3}}n_2 
+ 2$ for integers $n_1 \ge 2$ and $n_2\ge 3$. There are $n$ labeled unit discs 
residing in $\mathcal W$. Also for technical reasons, we assume that the discs 
are open, i.e., two discs are not in collision when their centers are exactly 
distance two apart. These discs may move in any direction with an instantaneous 
velocity $v$ satisfying $\vert v \vert \in [0, 1]$. Let $\mathcal C_f \subset 
\mathbb R^2$ denote the free configuration space for a single robot in $\W$.
The centers of the $n$ discs are initially located at $\S = \{s_1, \ldots, s_n\} 
\subset \mathcal C_f$, with goals $\G = \{g_1, \ldots, g_n\} \subset \mathcal 
C_f$. For all $1 \le i \le n$, a disc labeled $i$ initially located at $s_i$ 
must move to $g_i$. 

Beside planning collision-free paths, we want to optimize the resulting path 
quality by minimizing the {\em global task completion time}, also commonly 
known as the {\em makespan}. Let $P = \{p_1, \ldots, p_n\}$ denote a set of 
feasible paths with each $p_i$ a continuous function, defined as 
\begin{align}
p_i: [0, t_f] \to C_f, p_i(0) = s_i, p_i(t_f) = g_i,
\end{align}
the makespan objective seeks a solution that minimizes $t_f$, i.e., let 
$\mathcal P$ denote the set of all feasible solution path sets, the task is 
to find a set $P$ with $t_f(P)$ approaching the optimal solution 
\begin{align}\label{equation:minimum-time}
t_{min} := \min_{P \in \mathcal P} t_f(P). 
\end{align}

Positive separation between the labeled discs is necessary to render the 
problem feasible (regardless of optimality). In this work, we 
require the following clearance condition between a pair of $s_i$ 
and $s_j$ and a pair of $g_i$ and $g_j$:
\begin{align}
\forall 1 \le i, j \le n,\quad \parallel s_i - s_j \parallel > \frac{8}{3},\>\> \parallel 
g_i - g_j \parallel > \frac{8}{3}. \label{eq:separation}
\end{align}

For notational convenience, we denote the problem address in this work as 
the {\em Optimal Labeled Disc Routing} problem (\oldr). 
By assumption~\eqref{eq:separation} and assuming that the unit discs occupy
the vertices of a regular triangular grid, the 
discs may occupy 
$({\frac{1}{2}\pi *1^2})/(\frac{1}{2}\frac{8}{3}*\frac{4}{\sqrt{3}}) \approx 
51\%$
of the free space in the limit (to see that this is the case, we note 
that each equilateral triangle with side length $\frac{8}{3}$ contains half of
a unit disc; each corner of the triangle contains $\frac{1}{6}$ of a disc). 

\subsection{Workspace Discretization}
Similar to \cite{ErdLoz86,PeaClaMcp08,SolHal12,YuRus14RR}, we approach \oldr 
through first discretizing the problem, starting by embedding a discrete graph 
within $\W$. The assumption of $w = 4n_1 + 2$ and $h = 
\frac{4}{\sqrt{3}}n_2 + 2$  on the workspace dimensions allows the embedding 
of a triangular grid with side length of $\frac{4}{\sqrt{3}}$ in $\W$ 
such that the grid has $2n_1$ columns and about $n_2$ (zigzagging) rows of 
equilateral triangles, and a clearance of $1$ from $\partial\W$. An example is 
provided in Fig.~\ref{fig:tri-grid}. 

\begin{figure}[ht]
\begin{center}
\begin{overpic}[width={\ifoc 4.8in \else 2.66in \fi},tics=5]
{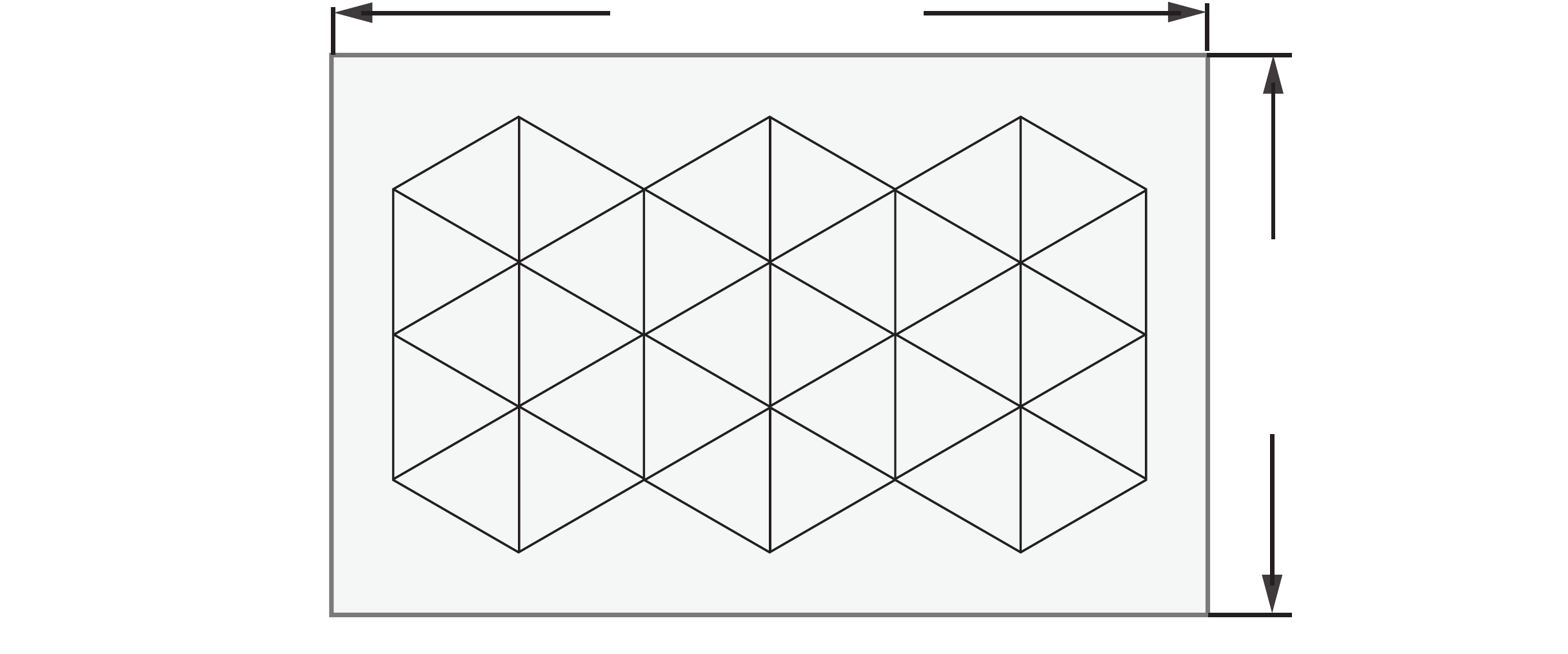}
\put(48,40.5){{\small $w$}}
\put(79.6,18){{\small $h$}}
\put(23,4){{\small $\W$}}
\end{overpic}
\end{center}
\caption{\label{fig:tri-grid} An example of a workspace $\W$ with $w = 
14$ and $h = 3\frac{4}{\sqrt{3}} + 2$, i.e., $n_1 = 3$ and $n_2 = 3$.
The embedded triangular grid is at least distance $1$ from the boundary of 
$\W$. The grid has $6$ columns and $2+$ zigzagging rows.}
\end{figure}

Throughout the paper, we denote the underlying graph of the triangular 
grid as $G$. Henceforth, we assume such a triangular grid $G$ for a given 
workspace $\W$. The choice of the side length of $\frac{4}{\sqrt{3}}$ 
for the triangular grid ensures that two unit discs located on adjacent 
vertices of $G$ may move simultaneously on $G$ without collision when the 
angle formed by the two traveled edges is not sharp (Fig.~\ref{fig:hex-tri}).
\begin{figure}[h]
\begin{center}
\begin{overpic}[width={\ifoc 4.8in \else 2.66in \fi},tics=5]
{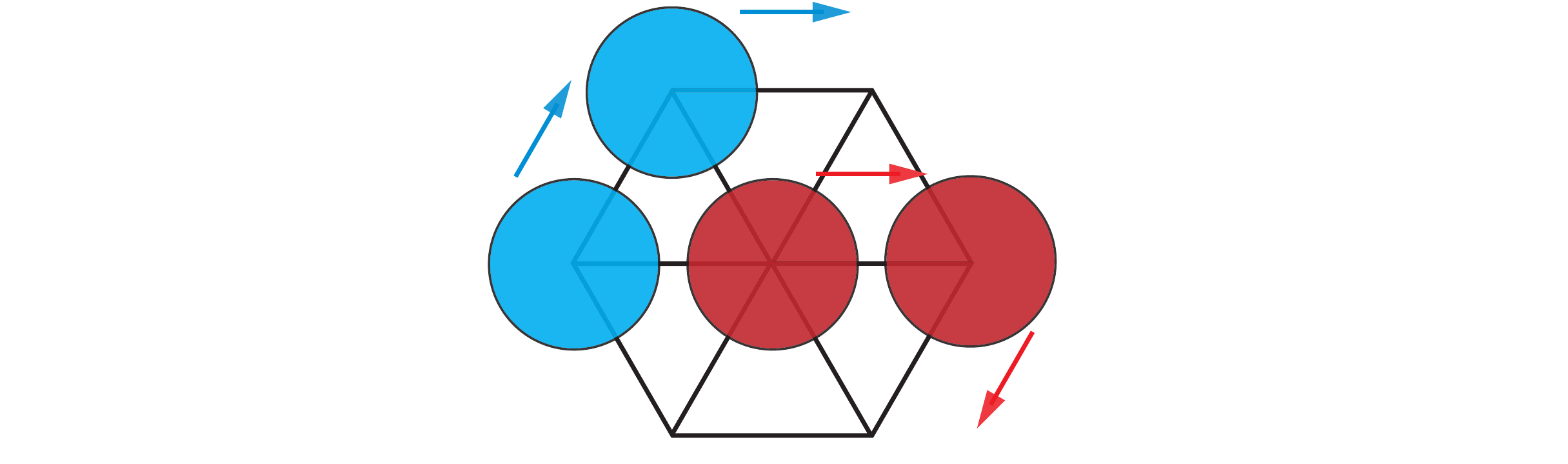}
\end{overpic}
\end{center}
\caption{\label{fig:hex-tri} On a triangular grid with a side length of 
$\frac{4}{\sqrt{3}}$, two unit discs, initially residing on two adjacent
vertices of the grid, may travel concurrently on the grid without collision 
when the two trajectories do not form a sharp angle. In the figure, the 
two cyan discs may travel as indicated without incurring collision. On the 
other hand, the red discs will collide if they follow the indicated 
travel directions.}
\end{figure}

We note that $w \ge 10$ and $h \ge 3\frac{4}{\sqrt{3}} + 2$ are needed for 
our algorithm to have completeness and optimality guarantees. For smaller $w$
or $h$, an instance may not be solvable. On the other hand, the discrete 
increment assumption on $w$ and $h$ are for technical convenience and are 
not strictly necessary. Without these discrete increments assumptions, we 
will need additional (and more complex) clearance assumptions between 
the discs and $\partial\W$, which does not affect the $51\%$ density bound  
since $\partial\W$ contributes $\Theta(w + h)$ to the area of $\W$ which is 
$wh$. The ratio is $\Theta(\frac{w + h}{wh})$ which goes to zero as both $w$ 
and $h$ increase. We also mention that, although this study only considers 
bounded rectangular workspace without static obstacles within the workspace, 
our results can be directly combined with \cite{YuRus14RR} to support static 
obstacles. 

\section{Translating Continuous Problems to Discrete Problems with Minimal 
Penalty on Optimality}\label{section:discretize}
A key insight enabling this work is that, under the separation 
condition~\eqref{eq:separation}, a continuous \oldr can be translated into a
discrete one with little optimality penalty. The algorithm for achieving this 
is relatively simple. For a given $\W$ and the corresponding $G = (V, E)$ 
embedded in $\W$, for each $s_i \in \S$, let $v_i^s \in V$ be a vertex of $G$ 
that is closest to $s_i$ (if there are more than one such $v_i^s$, pick an 
arbitrary candidate). After all $v_i^s$'s (let $V_{\S} = \{v_i^s\}$) are 
identified for $1 \le i \le n$, let $d_{\max} = \max_i \parallel v_i^s - 
s_i \parallel$. Note that $d_{\max} \le \frac{4}{3}$. We then let the labeled discs 
at $s_i$ move in a straight line to the corresponding $v_i^s$ at a constant 
speed given by 
$
\frac{\parallel v_i^s - s_i\parallel}{d_{\max}}
$, which means that for all $1 \le i \le n$, disc $i$ will reach $v_i^s$ in 
exactly one unit of time. The same procedure is then applied to $\G$ to obtain 
$V_{\G} = \{v_i^g\}$. The discrete \oldr is fully defined by $(G, V_{\S}, 
V_{\G})$. We denote the algorithm as \oldrdisc. Fig.~\ref{fig:to-discrete} 
illustrates the assignment of a few unit discs to vertices of the triangular 
grid. 
\begin{figure}[h]
\begin{center}
\begin{overpic}[width={\ifoc 4.8in \else 2.66in \fi},tics=5]
{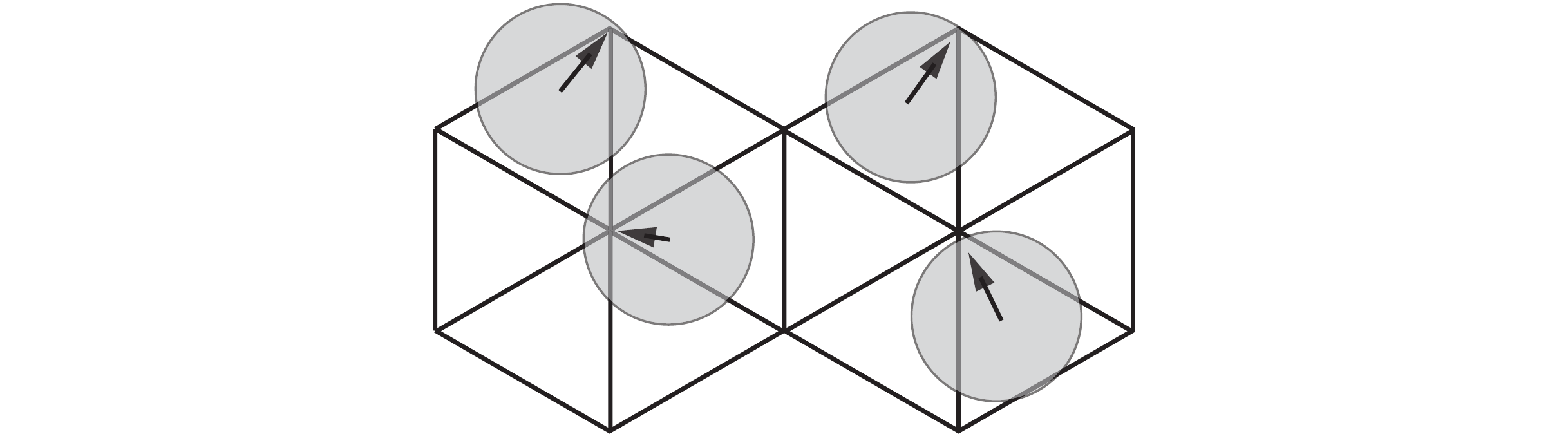}
\end{overpic}
\end{center}
\caption{\label{fig:to-discrete} An illustration of assigning a few unit 
discs to vertices of the triangular grid.}
\end{figure}

Because it takes a constant amount computational effort to deal with one disc,
\oldrdisc runs in linear time, i.e.,
\begin{proposition}
\oldrdisc has a running time of $O(n)$. 
\end{proposition}

The rest of this section is devoted to showing that \oldrdisc is 
collision-free and incurs little penalty on time optimality. We only need 
to show this for translating $\S$ to $V_{\S}$; translating $\G$ to $V_{\G}$ 
is a symmetric operation. We first make the straightforward observation that
\oldrdisc adds a makespan penalty of up to four because translating $\S$ to 
$V_{\S}$ takes exactly one unit of time. Same holds for translating $\G$ to 
$V_{\G}$. 
\begin{proposition}
\oldrdisc incurs a makespan penalty of up to four. 
\end{proposition}
We then show \oldrdisc assigns a unique $v_i^s \in V$ for a given 
$s_i \in S$.

\begin{lemma}\label{l:unique-assignment}
\oldrdisc assign a unique $v_i^s \in V$ for an $s_i \in S$.
\end{lemma}
\begin{proof}Each equilateral triangle in $G$ has a side length of 
$\frac{4}{\sqrt{3}}$, which means that the distance from the center of a 
triangle to its vertices is $\frac{4}{3}$. Therefore, for any $s_i \in S$, 
it must be at most of distance $\frac{4}{3}$ to at least one vertex of $G$.
Let this vertex be $v_i^s$. Now given any other $s_j \in \S$, assume 
\oldrdisc assigns to it $v_j^s$. We argue that $v_i^s \ne v_j^s$ because 
otherwise 
\[
\frac{4}{3} + \frac{4}{3}  \ge \parallel v_i - v_i^s \parallel + \parallel 
v_j - v_j^s \parallel = \parallel v_i - v_i^s \parallel + \parallel v_j - v_i^s \parallel 
\ge \parallel v_i - v_j \parallel > \frac{8}{3},
\]
which is a contradiction. Here, the first $\ge$ holds because $\parallel 
v_i - v_i^s \parallel \le \frac{4}{3}$ and $\parallel v_j - v_j^s \parallel 
\le \frac{4}{3}$ by \oldrdisc; the second $\ge$ is due to 
the triangle inequality. The $>$ is due to assumption ~\eqref{eq:separation}.  
~\qed
\end{proof}

Next, we establish that \oldrdisc is collision-free. For the proof, we use 
geometric arguments assisted with computation-based case analysis. 
\begin{theorem}
\oldrdisc guarantees collision-free motion of the discs.
\end{theorem}
\begin{proof}We fix a vertex $v \in V$ of the triangular grid $G$. By
Lemma~\ref{l:unique-assignment}, at most one $s_i \in \S$ may be matched with 
$v$, in which case $v$ becomes $v_i^s$. If this is the case, then $s_i$ must
be located within one of the six equilateral triangles surrounding $v$. Assume
with out loss of generality that $s_i$ belongs to an equilateral triangle 
$\triangle uvw$ as shown in Fig.~\ref{fig:collision-free}(a). The rules of 
\oldrdisc further imply that $s_i$ must fall within one (e.g., the orange 
shaded triangle in Fig~\ref{fig:collision-free} (a)) of the six triangles 
belonging to $\triangle uvw$ that are formed by the three bisectors of 
$\triangle uvw$. Let this triangle be $\triangle vox$. Now, let $s_j \ne s_i$ 
be the center of 
a labeled disc $j$; assume that disc $j$ go to some $v_j^s \in V$. By symmetry, 
if we can show that disc $i$ with $s_i \in \triangle vox$ and an arbitrary 
disc $j$ with $\parallel s_i - s_j\parallel > \frac{8}{3}$ will not collide 
with each other as disc $i$ and disc $j$ move along $s_iv_i^s$ and 
$s_jv_j^s$, respectively, then \oldrdisc is a collision-free procedure. 
\begin{figure}[h]
\begin{center}
\begin{overpic}[width={\ifoc 4.8in \else 2.66in \fi},tics=5]
{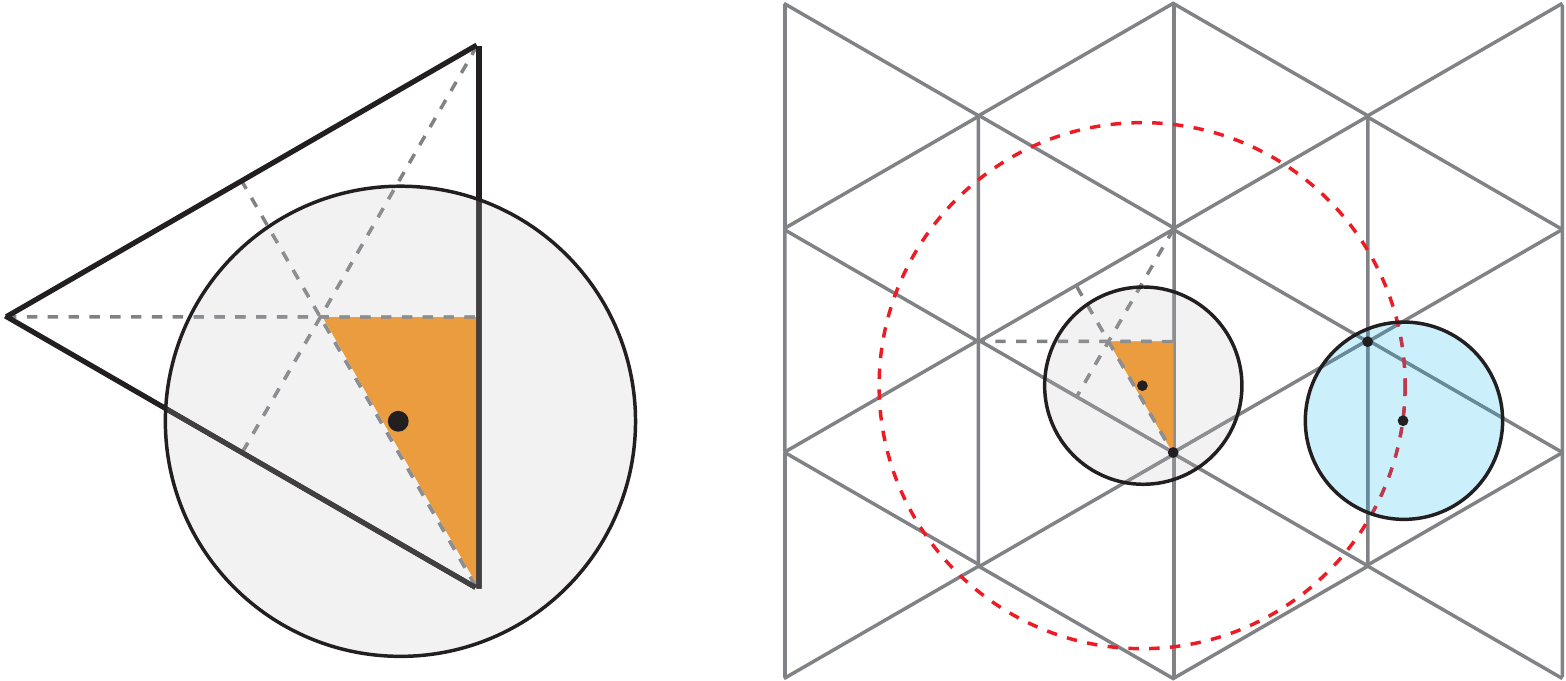}
\put(-3,22){{\small $u$}}
\put(32,41){{\small $w$}}
\put(22,3.5){{\small $v(v_i^s)$}}
\put(17,24){{\small $o$}}
\put(32,22){{\small $x$}}
\put(26,18){{\small $s_i$}}
\put(73.5,18.5){{\small $s_i$}}
\put(75.5,11){{\small $v_i^s$}}
\put(90.5,15){{\small $s_j$}}
\put(85,18.5){{\small $v_j^s$}}
\put(18, -4){{\small (a)}}
\put(72.5, -4){{\small (b)}}
\end{overpic}
\end{center}
\caption{\label{fig:collision-free} (a) By symmetry, for an $s_i \in \S$ to 
be moved to some $v = v_i^s$, we only need to consider the region $\triangle
vox$, which is $\frac{1}{12}$-th of all possible places where $s_i$ may appear. 
(b) For a fixed $s_i$, we only need to consider $s_j$ that is of exactly 
$\frac{8}{3}$ distance from it.}
\end{figure}
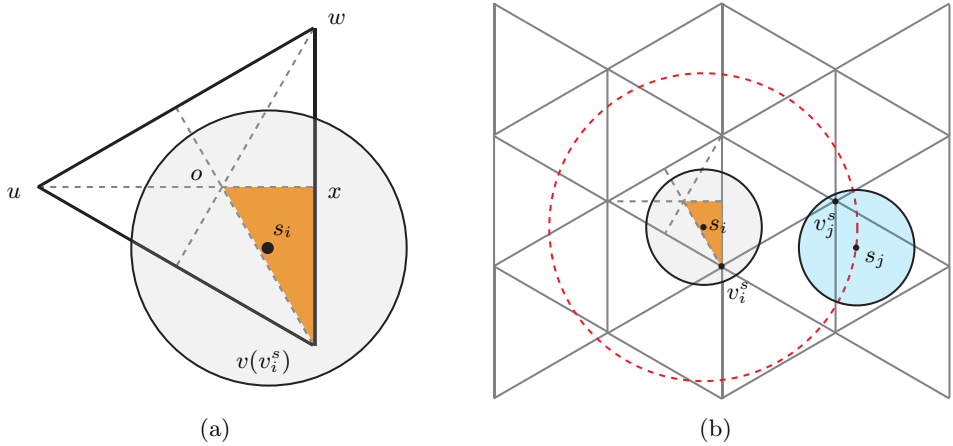

We then make the observation that, if disc $i$ and disc $j$ collide as we align
their centers to vertices of the triangular grid, at some point, the distance 
between their centers must be exactly $\frac{8}{3}$ before they may collide 
(when their centers are of distance less than $2$). Following this reasoning, 
instead of showing a disc $j$ with $\parallel s_i - s_j \parallel > 
\frac{8}{3}$ will not collide with disc $i$, it suffices to show the same 
only for $\parallel s_i - s_j \parallel = \frac{8}{3}$. That is, it is 
sufficient to show that, for any $s_i \in \triangle vox$ and any $s_j$ on a
circle of radius $\frac{8}{3}$ centered at $s_i$, disc $i$ and disc $j$ will 
not collide as $s_i$ and $s_j$ move to $v_i^s$ and $v_j^s$, respectively, 
according to the rules specified by \oldrdisc (see 
Fig.~\ref{fig:collision-free}(b) for an illustration).

To proceed from here, one may attempt direct case-by-case geometric analysis, 
which appears be quite tedious. We instead opt for a more direct computer 
assisted proof as follows. We first partition $\triangle vox$ using 
axis-aligned square grids with side length $\varepsilon$; $\varepsilon$ is 
some parameter to be determined through computation. For each of the 
resulting $\varepsilon \times \varepsilon$ square region (a small green square 
in Fig.~\ref{fig:discretize-square-grid}), we assume that $s_i$ is at its 
center. For each fixed $s_i$, an annulus centered at $s_i$ with inner radius 
$\frac{8}{3} - \frac{\sqrt{2}\varepsilon}{2}$ and outer radius $\frac{8}{3} + 
\frac{\sqrt{2}\varepsilon}{2}$ is obtained (part of which is illustrated as 
in Fig.~\ref{fig:discretize-square-grid}). Given this construction, for any 
potential $s_i'$ in a fixed $\varepsilon \times \varepsilon$ square, a circle 
of radius $\frac{8}{3}$ around it falls within the annulus. 

We then divide the outer perimeter of the annulus into arcs of length no more 
than $\sqrt{2}\varepsilon$. For each piece, we obtain a roughly square region 
with side length $\sqrt{2}\varepsilon$ on the annulus, one of which is shown 
as the red square in Fig.~\ref{fig:discretize-square-grid}. We take the center 
of the square as $s_j$. We then fix $v_j^s$ accordingly (note that it may be 
the case that $v_j^s$ is not unique for the square that bounds a piece of arc, 
in which case we will attempt all potentially valid $v_j^s$'s). 
\begin{figure}[h]
\begin{center}
\begin{overpic}[width={\ifoc 4.8in \else 2.66in \fi},tics=5]
{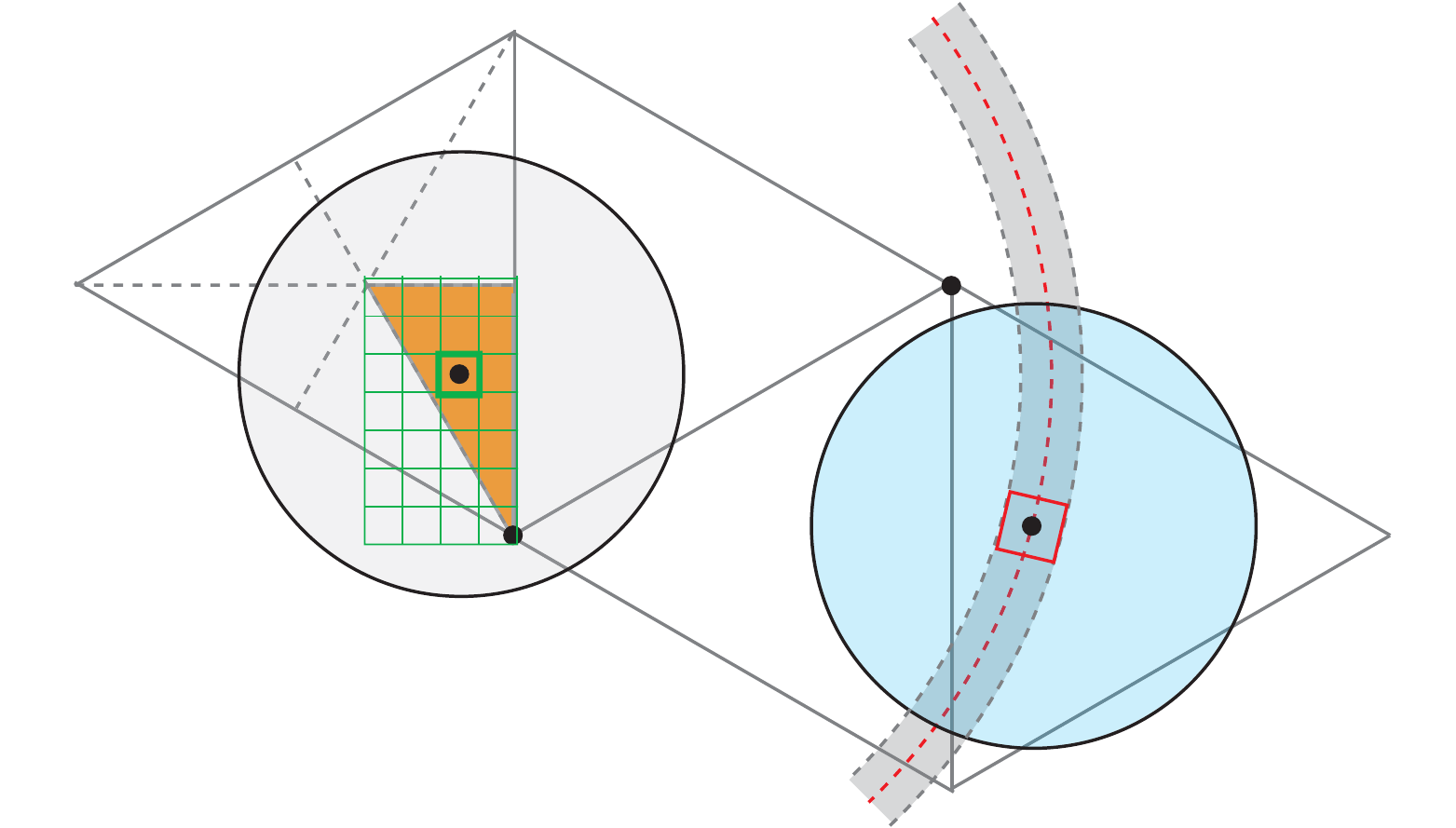}
\put(31,33.2){{\small $s_i$}}
\put(37,19.2){{\small $v_i^s$}}
\put(73,18){{\small $s_j$}}
\put(64,40){{\small $v_j^s$}}
\end{overpic}
\end{center}
\caption{\label{fig:discretize-square-grid} Illustration of picking a pair
of $s_i$ and $s_j$ for a computer based proof.}
\end{figure}
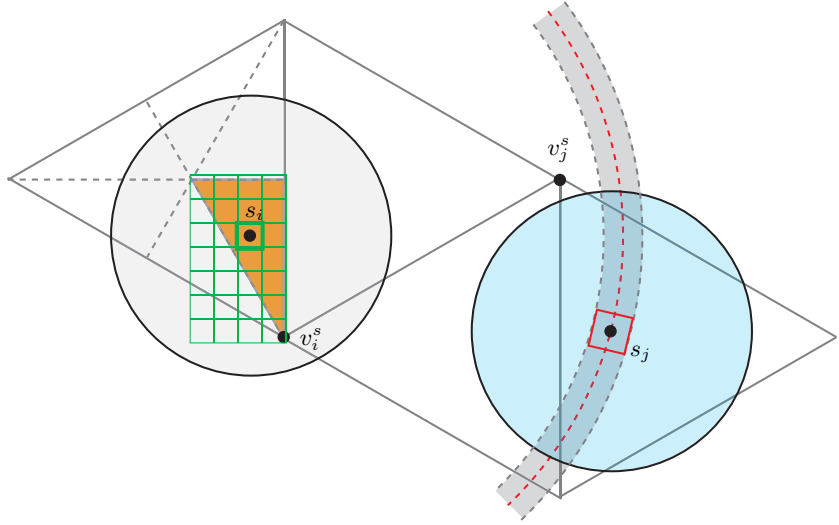

For each fixed set of $s_i, v_i^s, s_j$, and $v_j^s$, following the rules of 
\oldrdisc, we may (analytically) compute the shortest distance between the 
centers of disc $i$ and disc $j$ as disc $i$ is moved from $s_i$ to $v_i^s$ 
while disc $j$ is moved from $s_j$ to $v_j^s$. Let the trajectory followed by 
the two centers in this case be $\tau_i(t)$ and $\tau_j(t)$, respectively, 
with $0 \le t \le 1$ (as guaranteed by \oldrdisc), we may express the distance 
(for fixed $\varepsilon$, $s_i, v_i^s, s_j$, and $v_j^s$) as 
$\delta_{\varepsilon}(s_i, s_j) = \min_t \parallel \tau_i(t) - \tau_j(t)\parallel
$.
For any $s_i'$ that falls in the same $\varepsilon \times 
\varepsilon$ box as $s_i$, if disc $i$ is initially located at $s_i'$, let 
it follow a trajectory $\tau_i'(t)$ to $v_i^s$. We observe that $\parallel
\tau_i(t) - \tau_i'(t) \parallel \le \varepsilon$. 
This holds because as the center of disc $i$ moves from anywhere within the 
$\varepsilon \times \varepsilon$ box to $v_i^s$, $\parallel \tau_i(t) - 
\tau_i'(t) \parallel$ continuously decreases until it reaches to zero at 
$v_i^s$, which is the same for both $s_i$ and $s_i'$. Therefore, the initial 
uncertainty is the largest, which is no more than $\varepsilon$ because 
$\parallel s_i - s_i' \parallel \le \frac{\sqrt{2}\varepsilon}{2}$. The same 
argument applies to disc $j$, i.e., $\parallel \tau_j(t) - \tau_j'(t) \parallel 
\le \varepsilon$. Therefore, we have
\begin{align*}
\delta_{\varepsilon}(s_i, s_j) & = \min_t \parallel \tau_i(t) - \tau_j(t)\parallel  \\
& = \min_t \parallel \tau_i(t) - \tau_i'(t) + \tau_j'(t) - \tau_j(t) + 
\tau_i'(t) - \tau_j'(t) \parallel \\
& \le \min_t (\parallel \tau_i(t) - \tau_i'(t) \parallel + \parallel \tau_j'(t) - 
\tau_j(t) \parallel + \parallel \tau_i'(t) - \tau_j'(t) \parallel) \\
& \le 2\varepsilon + \min_t \parallel \tau_i'(t) - \tau_j'(t) \parallel \\
& \le 2\varepsilon + \delta_{\varepsilon}(s_i', s_j').
\end{align*}

If $\delta_{\varepsilon}(s_i, s_j) > 2\varepsilon$, then we may conclude that 
$\delta_{\varepsilon}(s_i', s_j') > 0$. We verify using a python program that 
gradually lowers $\varepsilon$ and compute the minimum 
$\delta_{\varepsilon}(s_i, s_j)$ over all possible choices of $s_i$ and $s_j$. 
When $\varepsilon = 0.025$, we obtain that $\delta_{\varepsilon}(s_i, s_j)$ 
is lower bounded at approximately $0.076$, which is larger than $2\varepsilon 
= 0.05$. Therefore, \oldrdisc is a collision-free algorithm. ~\qed
\end{proof}

With \oldrdisc, in Section~\ref{section:o1-opt-algorithm} and 
Section~\ref{section:ilp-model}, we assume a discrete multi-robot 
routing problem is given as a 3-tuple $(G, V_{\S}, V_{\G})$ in which 
$G$ is the unique triangular grid embedded in $\W$. Also, $V_{\S}, V_{\G} 
\subset V$ and $|V_{\S}| = |V_{\G}| = n$. 

\section{Constant-Factor Time-Optimal Multi-Robot Routing on Triangular Grid}
\label{section:o1-opt-algorithm}

In \cite{Yu18RSS}, it is established that constant-factor makespan time-optimal 
solution can be computed in quadratic running time on a $k$-dimensional 
orthogonal grid $G$ for an arbitrary fixed $k$. It is a surprising result that 
applies even when $n = |V|$, i.e., there is a robot or disc on every vertex of 
grid $G$. The functioning of the algorithm, \paf (standing for {\em partition 
and flow}), requires putting together many algorithmic techniques. However, the 
key requirements of the \paf algorithm hinges on three basic operations, which 
we summarize here for the case of $k = 2$. Due to limited space, only limited 
details are provided. 

First, to support the case of $n = |V|$ while ensuring desired optimality, it 
must be possible to ``swap'' two adjacent discs in a constant sized 
neighborhood in a constant number of steps (i.e. makespan), as illustrated in 
Fig.~\ref{fig:23}. This operation is essential in ensuring makespan time optimality 
as the locality of the operation allows many such operations to be concurrently 
carried out. 
\begin{figure}[h]
\begin{center}
\begin{overpic}[width={\ifoc 4.8in \else 3.42in \fi},tics=5]{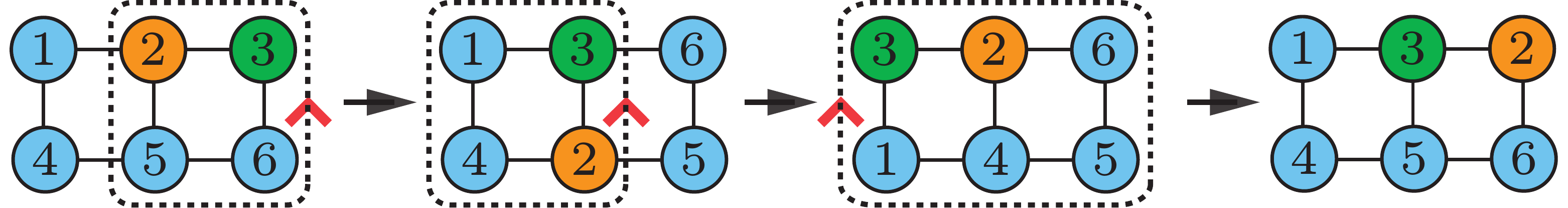}
\end{overpic}
\end{center}
\caption{\label{fig:23} Discs $2$ and $3$ may be ``swapped'' in three steps on 
a $3 \times 2$ grid, implying that any two discs can be swapped in $O(1)$ steps 
without net effect on other discs.}
\end{figure}

Second, it must be possible to iteratively {\em split} the initial problem into 
smaller sub-problems. This is achieved using a {\em grouping} operation that 
in turn depends on the swap operation. We illustrate the idea using an example. In 
Fig.~\ref{fig:split}(a), a $8 \times 4$ grid is split in the middle into two 
smaller grids. Each vertex is occupied by a disc; we omit the individual labels. 
The lightly (cyan) shaded discs have goals on the right $4\times 4$ grid. The 
grouping operation moves the $7$ lightly shaded discs to the right, which also 
forces the $7$ darker shaded discs on the right to the left side. This is 
achieved through multiple rounds of concurrent swap operations either along 
horizontal lines or vertical lines. The result is Fig.~\ref{fig:split}(b). This 
effectively reduces the initial problem $(G, V_{\S},V_{\G})$ to two disjoint 
sub-problems. Repeating the iterative process can actually solve the problem 
completely but does not always guarantee constant-factor makespan time 
optimality in the worst case. This is referred to as the \isag algorithm in 
\cite{Yu18RSS}.
\begin{figure}[h]
\begin{center}
\begin{overpic}[width={\ifoc 6in \else 3.04in \fi},tics=5]{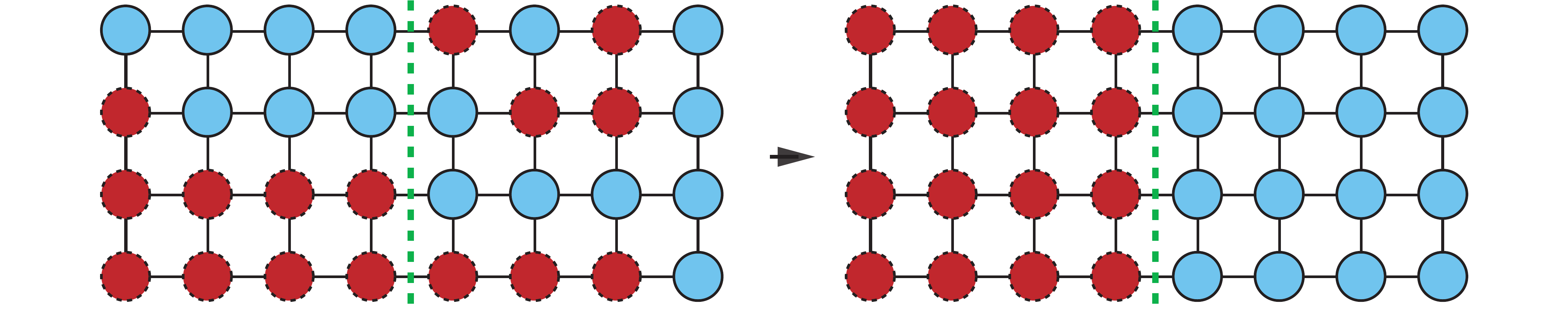}
\put(25, -2.5){{\small (a)}}
\put(72.5, -2.5){{\small (b)}}
\end{overpic}
\end{center}
\caption{\label{fig:split}Illustration of an iteration of the \isag algorithm.}
\end{figure}

Lastly, \paf achieves guaranteed constant-factor optimality using \isag as a 
subroutine. It begins by computing the maximum distance between any pair of 
$v_i^s \in V_{\S}$ and $v_i^g \in V_{\G}$ over all $1 \le i \le n$. Let this 
distance be $d_g$. $G$ is then partitioned into square grid {\em cells} of size 
roughly $5d_g \times 5d_g$ each. With this partition, a disc must have its goal 
in the same cell it is in or in a neighboring cell. After some pre-processing 
using \isag, the discs that need to cross cell boundaries can be arranged to 
be near the destination cell boundary. At this point, multiple global 
circulations (a circulation may be interpreted as discs rotating synchronously 
on a cycle on $G$） are arranged so that every disc ends up in a $5d_g \times 
5d_g$ cell partition where its goal also resides. A rough illustration of the 
global circulation concept is provided in Fig.~\ref{fig:circulation}. Then, 
a last round of \isag is invoked at the cell level to solve the problem, which 
yields a constant-factor time-optimal solution even in the worst case. 
\begin{figure}[h]
\begin{center}
\begin{overpic}[width={\ifoc 6in \else 3.04in \fi},tics=5]{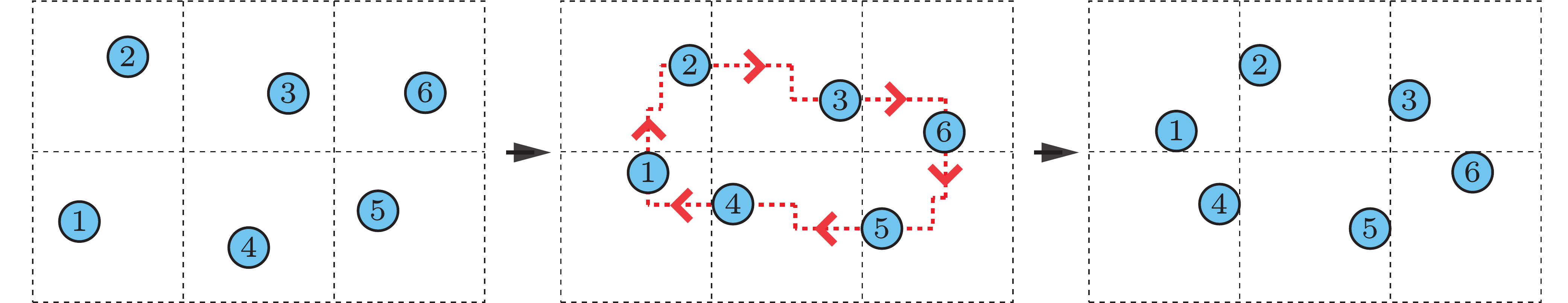}
\put(2.5, 17){{\small $1$}}
\put(12, 17){{\small $2$}}
\put(21.5, 17){{\small $3$}}
\put(2.5, 0.5){{\small $4$}}
\put(12, 0.5){{\small $5$}}
\put(21.5, 0.5){{\small $6$}}
\put(15.5, -3){{\small (a)}}
\put(45.5, -3){{\small (b)}}
\put(82, -3){{\small (c)}}
\end{overpic}
\end{center}
\caption{\label{fig:circulation}Illustration of a single global circulation 
constructed and executed by \paf (the discs and the underlying grid cells are 
not fully drawn). (a) In six partitioned ($5d_g \times 5d_g$) cells numbered 
$1-6$ in $G$, there are six labeled discs with goals in the correspondingly 
numbered cells, e.g., disc $1$ should be in cell $1$. (b) Using \isag in each 
cell, discs $1-6$ are moved to boundary areas and a cycle is formed on $G$ for 
robot routing. (c) Moving all discs on the cycle by one edge synchronously, 
all discs are now in the desired cell; no other discs (not shown) have crossed 
any cell boundary.}
\end{figure}
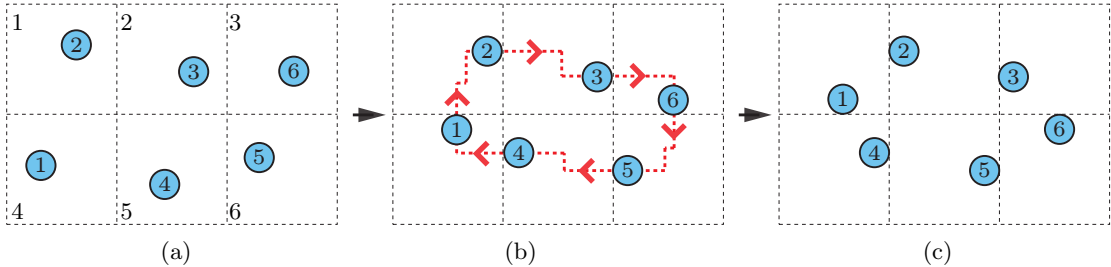

To adapt \paf to the special triangular grid graph $G$, we need to: 
{\em (i)} identify a constant sized local neighborhood for the swapping 
operation to work, {\em (ii)} identify two ``orthogonal'' directions that 
cover $G$ for the \isag algorithm to work, and {\em (iii)} ensure that 
the constructed global circulation can be executed. Because of the limitation 
imposed by the triangular grid, i.e., any two edges of a triangle cannot be 
used at the same time (see Fig.~\ref{fig:hex-tri}), achieving these conditions
simultaneously becomes non-trivial. In what follows, we will show how we may 
{\em simulate} \paf on a triangular grid $G$ under the assumption that all 
vertices of $G$ are occupied by labeled discs, i.e., $n = |V|$. For the case 
of $n < |V|$, we may treat empty vertices as having ``virtual discs'' placed 
on them. 

Because two edges of a triangle cannot be simultaneously used, we use two 
adjacent hexagons on $G$ (e.g., the two red full hexagons in 
Fig.~\ref{fig:swapping}(a)) to simulate the two square cells in Fig.~\ref{fig:23}. 
It is straightforward to verify that the swap operation can be carried out 
using two adjacent hexagons. There is an issue, however, as not all vertices
of $G$ can be covered with a single hexagonal grid. For example, the two red 
hexagons in Fig.~\ref{fig:swapping}(a) left many vertices uncovered. This can be 
resolved using up to three sets of interleaving hexagon grids as illustrated 
in Fig.~\ref{fig:swapping}(a) (here we use the assumption that $\W$ has dimensions 
$w \ge 10$ and $h \ge 3 \frac{4} {\sqrt{3}} + 2$, which limits the possible 
embeddings of the triangular grid $G$). We note that for the particular graph 
$G$ in Fig.~\ref{fig:swapping}(a), we only need the red and the green hexagons to 
cover all vertices. 
\begin{figure}[h]
\begin{center}
\begin{overpic}[width={\ifoc 4.8in \else 3.04in \fi},tics=5]{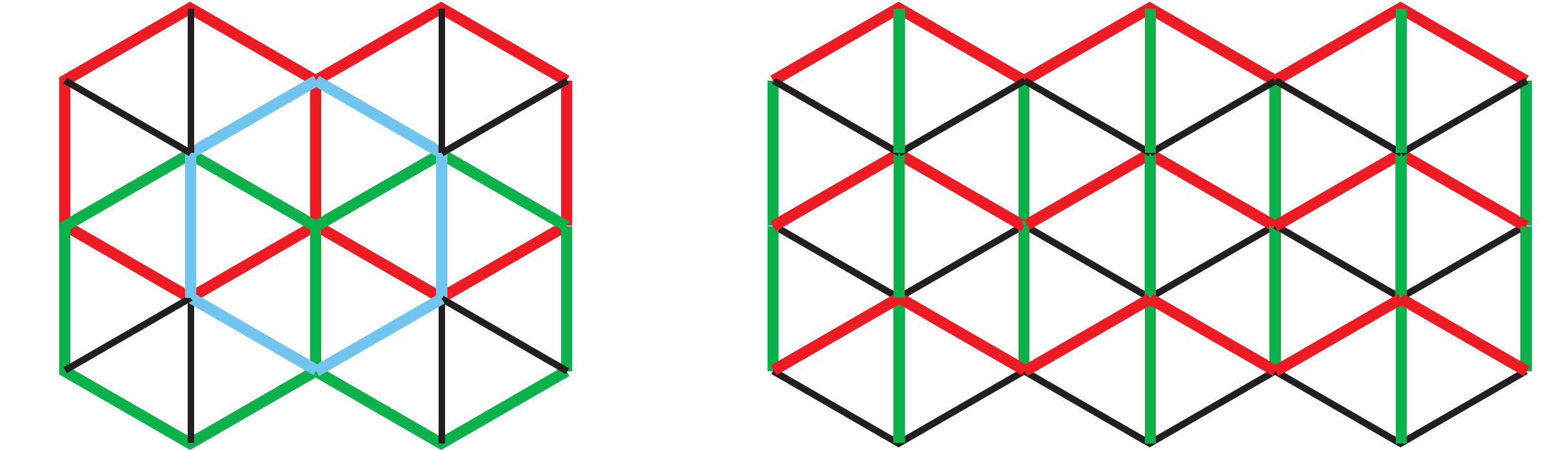}
\put(18, -4){{\small (a)}}
\put(71, -4){{\small (b)}}
\end{overpic}
\end{center}
\caption{\label{fig:swapping}(a) We may use the red, green, and cyan hexagon 
grids on $G$ to perform the swap operation. (b) The red and green paths 
may serve as orthogonal paths for carrying out the split and group operations
as required by \isag.}
\end{figure}

To realize requirement {\em (ii)}, i.e., locating two sets of ``orthogonal'' 
paths for carrying out \isag iterations, we may use the red and green paths as 
illustrated in Fig.~\ref{fig:swapping}(b). The remaining issue is that the red 
waving paths do not cover the few vertices at the bottom of $G$ (the green 
paths, on the other hand, covers all vertices of $G$). This issue can be 
addressed with some additional swaps (e.g., with a second pass) which 
still only takes constant makespan during each iteration of \isag and does not 
impact the time optimality or running time of \isag. 

The realization of requirement {\em (iii)} is straightforward as the only 
restriction here is that the closed paths for carrying out circulations on $G$ 
cannot contain sharp turns. We can readily realize this using any one of the 
three interleaving hexagonal grids on $G$ that we use for the swap operation, 
e.g., the red one in Fig.~\ref{fig:swapping}(a). Clearly, any cycle on a 
hexagonal grid can only have angles of $\frac{2\pi}{3}$ which are obtuse. We 
note that there is no need to cover all vertices for this global 
circulation-based routing operation because only a fraction ($< \frac{1}{2}$, 
see \cite{Yu18RSS} for details) of discs need to cross the $5d_g \times 5d_g$ 
cell boundary. On the other hand, any one of the three hexagonal grids cover 
about $\frac{2}{3}$ of the vertices on a large $G$. 


Calling the adapted \paf algorithm on the special triangular grid as \paft, 
we summarize the discussion in this section in the following result. 
\begin{lemma}\label{t:paft}\paft computes constant-factor makespan 
time-optimal solutions for multi-robot routing on triangular grids in 
$O(|V|^2)$ time. 
\end{lemma}

Combining \oldrdisc with \paft then gives us the following. In deriving the 
running time result, we use the fact that $wh = \Theta(|V|) = \Omega(n)$. 

\begin{theorem}In a rectangular workspace $\W$ with $w \ge 10$ and $h \ge 
3\frac{4}{\sqrt{3}} + 2$, for $n$ labeled unit discs with start and goal 
configurations with separation over $\frac{8}{3}$, constant-factor 
makespan time-optimal collision-free paths connecting the two configurations 
may be computed in $O(w^2h^2)$ time.
\end{theorem}

We conclude this section with the additional remark that \paft should mainly 
be viewed as providing a theoretical guarantee than being a practical algorithm 
due to the fairly large constant in the optimality guarantee. 

\section{Fast Computation of Near-Optimal Solutions via Integer Linear 
Programming}\label{section:ilp-model}
From the practical standpoint, the \oldrdisc algorithm opens the possibility 
for plugging in any discrete algorithm for multi-robot routing. Indeed, 
algorithms including these from 
\cite{StaKor11,WagChoC11,boyarski2015icbs,cohen2016improved,YuLav16TRO}
may be modified to serve this purpose. In this paper, we develop a new 
integer linear programming (ILP) approach based on a time-expanded network
structure proposed in \cite{YuLav16TRO}. The benefit of using an ILP model 
is its high-level of flexibility and high computational performance when 
combined with appropriate solvers, e.g., Gurobi \cite{gurobi}. 

\subsection{Integer Linear Programming Model for Multi-Robot Routing on 
Triangular Grids}
The essential idea behind an ILP-based approach, e.g., \cite{YuLav16TRO}, is 
the construction of a directed time-expanded network graph representing
the possible flow of the robots over time. Given a discrete problem instance 
$(G, V_{\S}, V_{\G})$, the network is constructed by taking the vertex set $V$ 
of $G$ and making $T  + 1$ copies of it. Each copy represent an integer 
time instance starting from $0$ to $T$. Then, a directed edge is added between 
any two vertices when they are both adjacent on $G$ and in time, in the 
direction from time step $t$ to time step $t + 1$. 

To build the ILP model, for each robot and each edge (which is represented 
as the combination of a starting vertex, an end vertex $j$, and a time step 
$t$), a binary variable is created to represent whether the given robot uses 
that edge as part of its trajectory. Constraints are then added to make sure 
that no collision between any two robots could occur. The basic model 
from \cite{YuLav16TRO} only ensures that no two robot can use the same edge 
or vertex at the same time. In our case, more complex interactions must be 
considered, which is detailed as follows. 

Denoting $N(i)$ as the set of vertex $i \in V$ and its neighbors, 
the ILP model contains two sets of binary variables: {\em (i)} 
$\{x_{r, i, j, t} | 1 \leq r \leq n, i \in V, j \in N(i), 0 \leq t < T\}$, 
where $x_{r, i, j, t}$ indicates whether robot $r$ moves from vertex $i$ to 
$j$ between time step $t$ and $t + 1$. Note that by reachability test, some 
variables here are fixed to $0$. {\em (ii)} $\{x_{r, v_r^g, 
v_r^s, T} | 1 \leq r \leq n\}$ which stands for virtual edges between the goal 
vertex of each robot at time step $T$ and its start vertex at time step $0$. 
$x_{r, v_r^g, v_r^s, T}$ is set to $1$ {\em iff} $r$ reaches its goal 
at $T$. The objective of this ILP formulation is to maximize the 
number of robots that reach their goal vertices at $T$, i.e.,
\[\text{maximize} \sum_{1 \leq r \leq n} x_{r, v_r^g, v_r^s, T}\]
under the constraints
\begin{align}
    &\forall 1 \leq r \leq n, 0 \leq t < T,  
    \sum_{i \in N(j)} x_{r, i, j, t} = \sum_{k \in N(j)} x_{r, j, k, t + 1} 
    \label{equation:ilp-constraint-1} \\ 
    &\forall 1 \leq r \leq n,  
    \sum_{i \in N(v_r^s)} x_{r, v_r^s, i, 0}  
    = \sum_{i \in N(v_r^g)} x_{r, i, v_r^g, T - 1} = x_{r, v_r^g, v_r^s, T} 
    \label{equation:ilp-constraint-2} \\ 
    &\forall 0 \leq t < T, i \in V, 
    \sum_{1 \leq r \leq n} \sum_{j \in N(i)} x_{r, i, j, t} \leq 1.
    \label{equation:ilp-constraint-3} \\ 
    &\forall 0 \leq t < T, i \in V, j \in N(i), 
    \sum_{1 \leq r \leq n} x_{r, i, j, t} + 
    \sum_{1 \leq r \leq n} x_{r, j, i, t} \leq 1.
    \label{equation:ilp-constraint-4}     
\end{align}
Here, constraint (\ref{equation:ilp-constraint-1}) and 
(\ref{equation:ilp-constraint-2}) ensure a robot always starts from its start 
vertex, and can only stay at the current vertex or move to an adjacent vertex in 
each time step. Moreover, constraint (\ref{equation:ilp-constraint-2}) is 
essential for objective value calculation. 
Constraint (\ref{equation:ilp-constraint-3}) avoids robots from simultaneously 
occupying the same vertex, while constraint (\ref{equation:ilp-constraint-4}) 
eliminates head-to-head collisions on edges. 

For a triangular grid, one extra set of constraints must be imposed so that 
any two robots cannot simultaneously move on the same triangle. Denote 
$\angle_{ijk}$ as a sharp angle formed by edges $(i, j), (j, k) \in E$, and 
$\mathcal A$ as the set of all such angles in $G$, the constraint can be 
expressed as 
\begin{align}\label{equation-ilp-angle}
    \forall 0 \leq t < T, \angle{ijk} \in \mathcal A, \sum_{1 \leq r \leq n} 
    (x_{r, i, j, t} + x_{r, j, i, t} + x_{r, j, k, t} + x_{r, k, j, t}) \leq 1,  
\end{align}
which may be more compactly as (which also reduce the number of constraints):
\begin{align}\label{equation-ilp-triangle}
    \sum_{1 \leq r \leq n} (x_{r, i, j, t} + x_{r, j, i, t} + x_{r, i, k, t} 
    + x_{r, k, i, t} + x_{r, j, k, t} + x_{r, k, j, t}) \leq 
    \lfloor 3 / 2 \rfloor = 1.
\end{align}

Building on the ILP model, the overall route planning algorithm for triangular
grids, \triilp, is outlined in Alg. \ref{algo:ilp}. In line \ref{algo:ilp-T}, 
an underestimated makespan $T$ is computed by routing robots to goal vertices 
while ignoring mutual collisions. Then, as $T$ gradually increases (line 
\ref{algo:ilp-increment}), ILP models are iteratively constructed and solved (line 
\ref{algo:ilp-pre}-\ref{algo:ilp-solve}) until the resulting objective value 
$objval$ equals to $n$. In line \ref{algo:ilp-ret}, time-optimal paths are 
extracted and returned. Derived from \cite{YuLav16TRO}, \triilp has 
completeness and optimality guarantees. 

To improve the scalability of the ILP-based algorithm, a {\em $k$-way split 
heuristic} is introduced in \cite{YuLav16TRO} that adds intermediate 
robot configurations (somewhere in between the start and goal configurations)
to split the problem into sub-problems. These sub-problems require fewer 
steps to solve, which means that the corresponding ILP models are much smaller 
and can be solved much faster. This heuristic is directly applicable to \triilp.
\begin{algorithm}
    \small 
	\DontPrintSemicolon 
    $T \gets ${\sc UnderestimatedMakespan}$(G, V_{\S}, V_{\G})$\; \label{algo:ilp-T}
    \While{True}{
    $model \gets$ {\sc PrepareModel}$(G, V_{\S}, V_{\G}, T)$\; \label{algo:ilp-pre}
    $objval \gets$ {\sc Optimize}$(model)$\; \label{algo:ilp-solve}
    \lIf{objval {\normalfont equals to} $n$}
    {\Return {\sc ExtractSolution}$(model)$} \label{algo:ilp-ret}
    \lElse{$T \gets T + 1$} \label{algo:ilp-increment}
    }
	\caption{{\sc \triilp}} 
	\label{algo:ilp} 
\end{algorithm}

\subsection{Performance Evaluation}
We evaluate the performance of \triilp based on two standard measures: {\em 
computational time} and {\em optimality ratio}. To compute the optimality 
ratio, we first obtain the {\em underestimated makespan} $\hat{t_i}$ number 
of steps to move robots to their goals, ignoring potential robot-robot 
collisions, for a given problem instance $i$. Denoting $t_i$ as the makespan 
produced by \triilp of the $i$-th problem instance, the optimality ratio is 
defined as $(\sum_i t_i) / (\sum_i \hat{t_i})$. For each set of problem 
parameters, ten random instances are generated and the average is taken. All 
experiments are executed on an Intel\textsuperscript{\textregistered} 
Core\textsuperscript{TM} i7-6900K CPU with 32GB RAM at 2133MHz. For the ILP 
solver, Gurobi 8 is used \cite{gurobi}. 

We begin with \triilp on purely discrete multi-robot routing problems. On 
a densely occupied minimum triangular grid ($n_1 = 2, n_2 = 3, |V| = 22, n 
= 16$) as allowed by our formulation, a randomly generated problem can be 
solved optimally within $5$ seconds on average. For a much larger environment 
($n_1 = 7, n_2 = 16, |V| = 232)$, we evaluate \triilp with $k$-way split 
heuristics, gradually increasing the number of robots. As shown in Fig. 
\ref{fig:discrete-result}, \triilp could solve problems with $50$ robots 
\begin{figure}[h]
    \centering
    \includegraphics[width=0.9\linewidth]{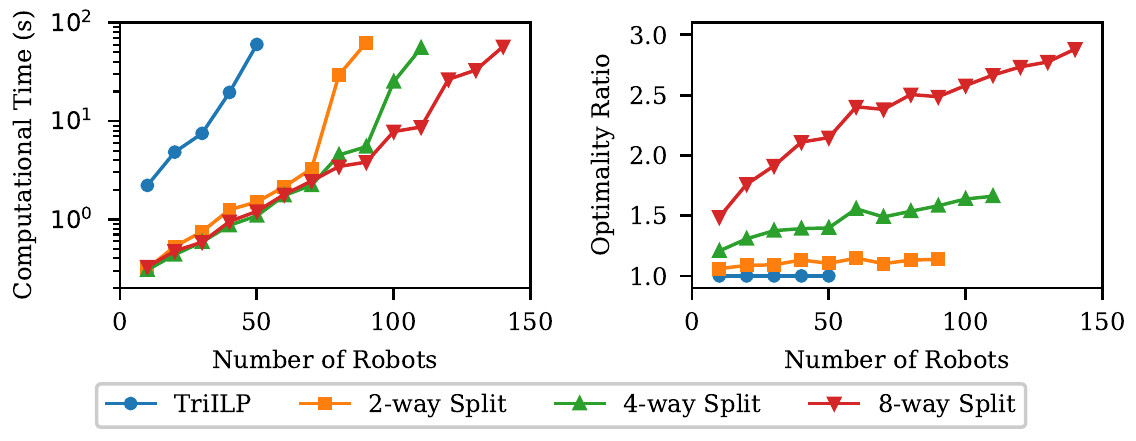}
    \caption{
        Performance of \triilp with $k$-way split heuristics on a triangular 
				grid with $232$ vertices and varying numbers of robots.}
    \label{fig:discrete-result}
\end{figure}
optimally in $60$ seconds. Performance of \triilp is significantly improved 
with $k$-way split heuristic: with $4$-way split, \triilp can solve problems 
with $110$ robots in $55$ seconds to $1.65$-optimal. With $8$-way split, we 
can further push to $140$ robots with reasonable optimality ratio. 

Solving (continuous) \oldr requires both \oldrdisc and \triilp. We first 
attempted a scenario of which the density approaches the theoretical limit 
by placing the robots just $\frac{8}{3}$ apart from each other in a regular 
(triangular) pattern for both start and goal configurations (see 
Fig.~\ref{fig:continuous-example}(a) for an illustration; we omit the labels 
of the robots, which are different for the start and goal configurations). 
After running \oldrdisc, we get a discrete arrangement as illustrated in 
Fig.~\ref{fig:continuous-example}(b). For this particular problem, we can 
compute a $1.5$-optimal solution in $2.1$ second without using splitting 
heuristics. 
\begin{figure}[h]
    \centering
		\begin{tabular}{ccc}
    \includegraphics[width=0.3\linewidth]{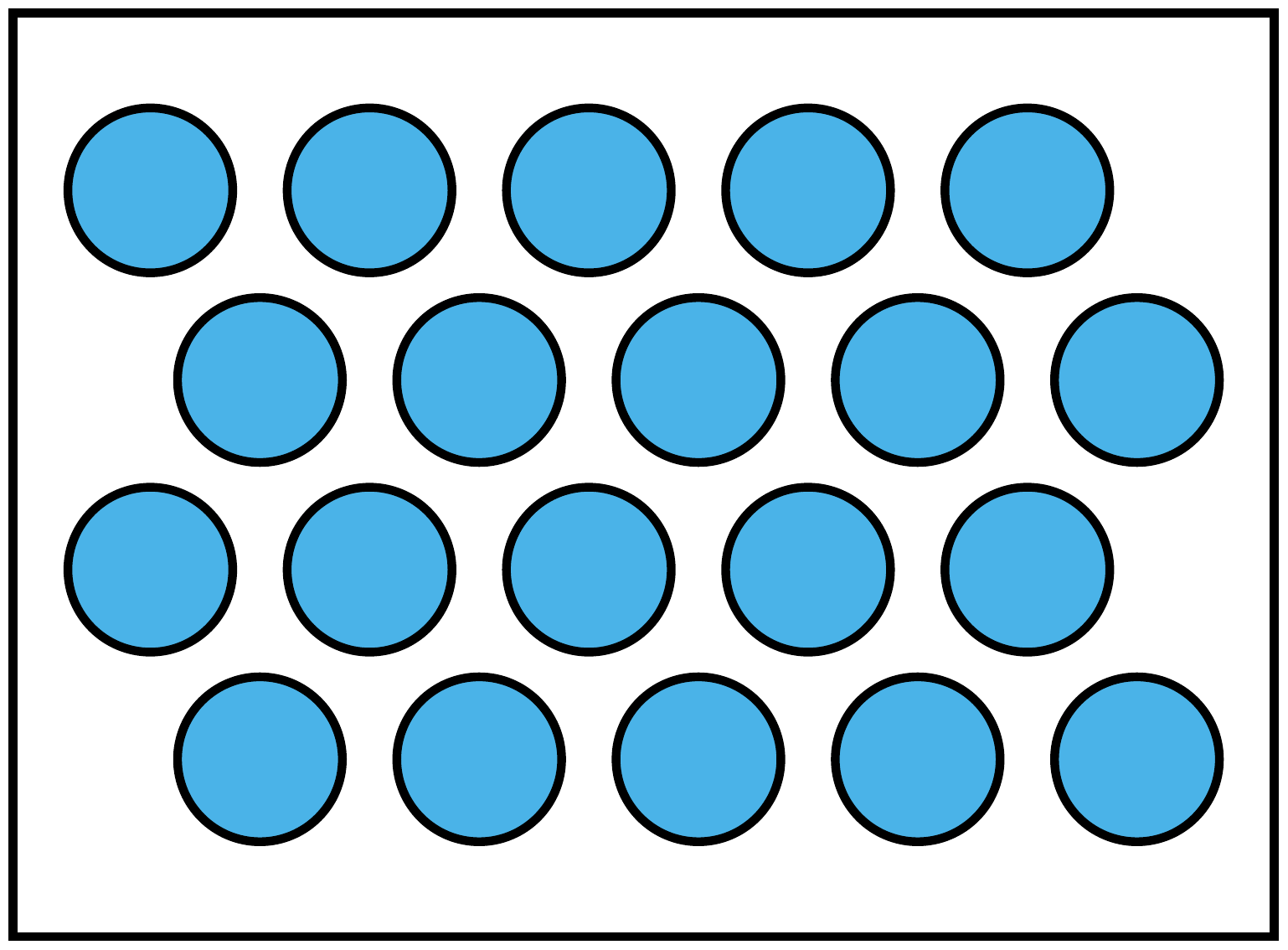}
		& \hspace{15mm} &
    \includegraphics[width=0.3\linewidth]{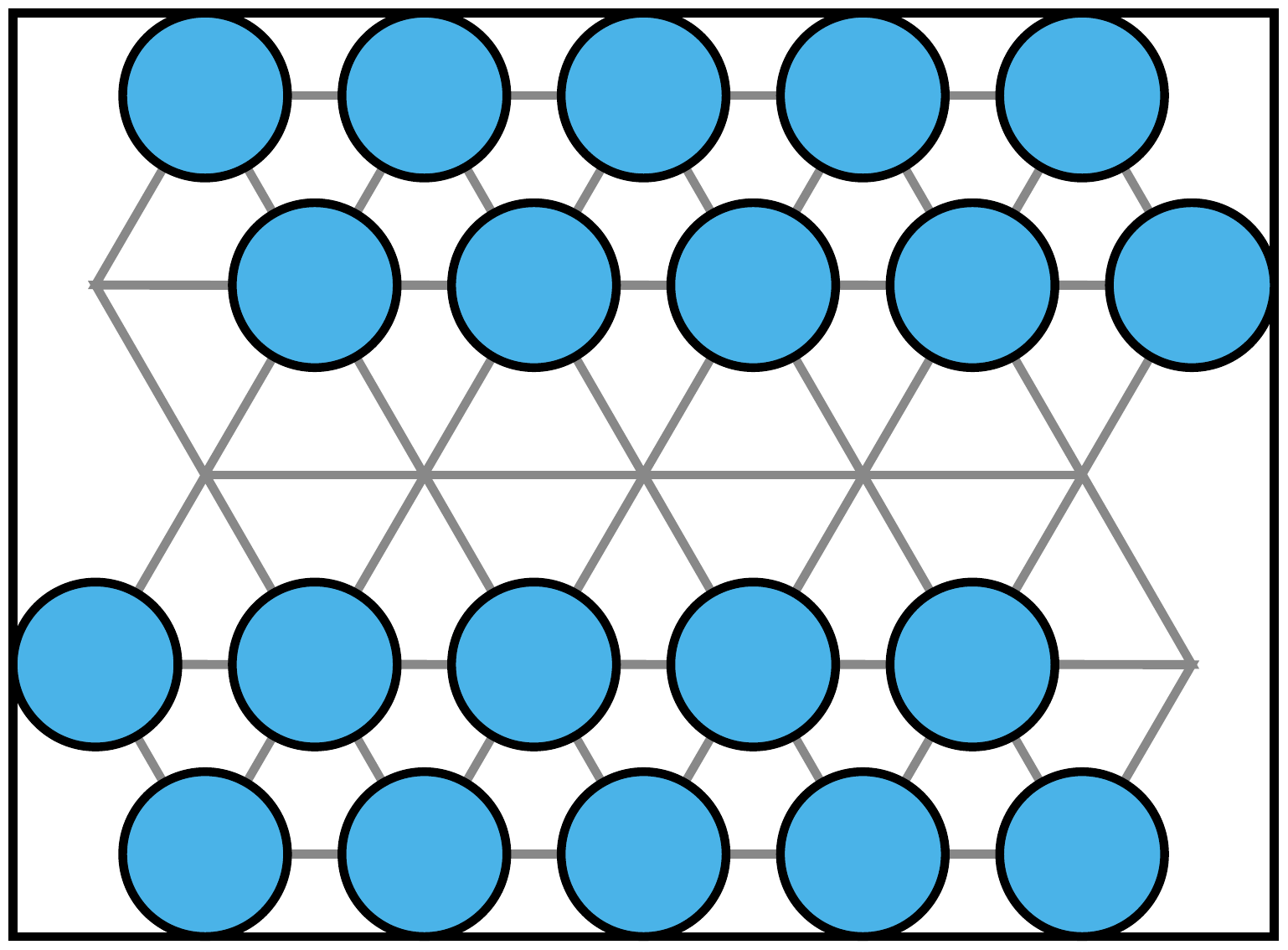}
		\\ (a) & & (b)
		\end{tabular}
    \caption{Illustration of a compact \oldr instance with $20$ densely packed 
    robots, and the configuration of robots after \oldrdisc.}
    \label{fig:continuous-example}
\end{figure} 

To test the effectiveness of combining \oldrdisc and \triilp, we 
constructed many instances similar to Fig.~\ref{fig:continuous-example} but 
with different environment sizes, always packing as many robots as possible 
with separation of exactly $\frac{8}{3}$. The computational performance of
this case is compiled in Fig.~\ref{fig:continuous-result}. With the $8$-way 
split heuristic, our method can solve tightly packed problems of $120$ robots 
in $21.93$ seconds with a $3.88$ optimality ratio. We note that the 
(underestimated) optimality ratio in this case actually decreases as the 
number of robots increases. This is expected because when the number of robots 
are small, the corresponding environment is also small. The optimality loss 
due to discretization is more obvious when the environment is smaller. 
\begin{figure}[h]
    \centering
    \includegraphics[width=0.9\linewidth]{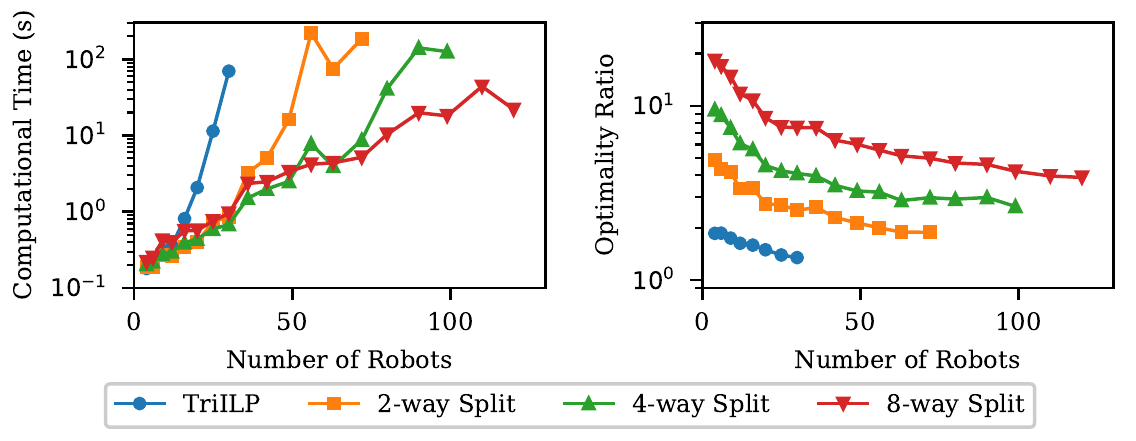}
    \caption{
        Performance of \triilp (plus \oldrdisc) on dense \oldr instances.}
    \label{fig:continuous-result}
\end{figure}

A second evaluation of \oldr carries out a comparison between \triilp (plus 
\oldrdisc) and \hexilp (the main algorithm from \cite{YuRus14RR}, which is 
based on a hexagonal grid discretization). We fix $\W$ with $w = 42$ and 
$h = 43.57$; the number of vertices in the triangular grid and hexagonal grid 
are $312$ and $252$, respectively. For each fixed number of robots $n$, $\S$ 
and $\G$ are randomly generated within $W$ that are at least $\frac{8}{3}$ 
apart. Note that this means that collisions may potentially happen for \hexilp 
during the discretization phase, which are ignored (to our disadvantage). The 
evaluation result is provided in Fig.~\ref{fig:comparison-result}. Since 
discretization based on triangular grid produces larger models, the running 
time is generally a bit higher when compared with discretization based on 
hexagonal grids. However, \triilp can solve problems with many more robots 
and also produce solutions with much better optimality guarantees.  
\begin{figure}[h]
    \centering
    \includegraphics[width=0.9\linewidth]{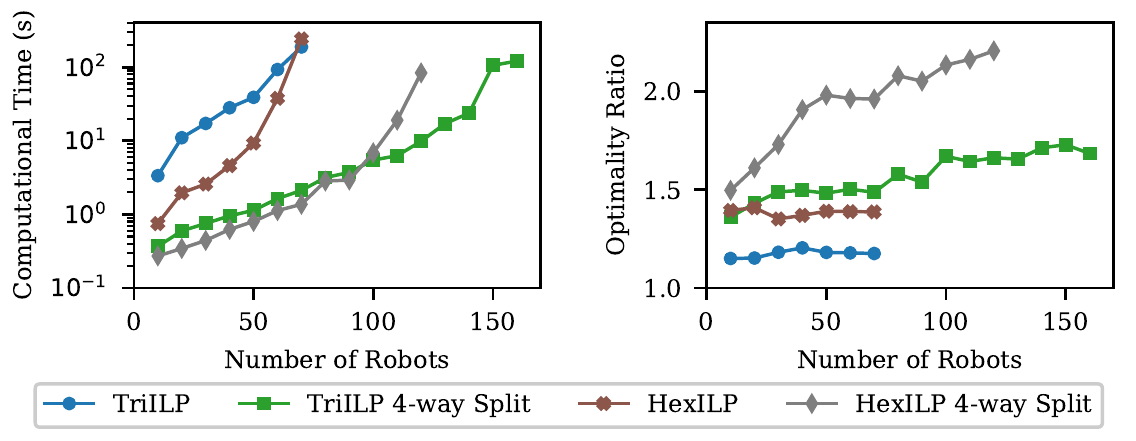}
    \caption{
        Performance comparison between \triilp and \hexilp (with and without 
				$4$-way split heuristic) on randomly generated \oldr instances with a 
				fixed $\W$.}
    \label{fig:comparison-result}
\end{figure}

\section{Conclusion and Future Work}\label{section:conclusion}
In this work, we have developed a complete, polynomial-time algorithm for
multi-robot routing in a bounded environment under extremely high robot 
density. The algorithm produces plans that are constant-factor time-optimal. 
A fast and more practical ILP-based algorithm capable of generating 
near-optimal solutions is also provided. We mention here that extensions to 
3D settings, which may be more applicable to drones and other airborne robot 
vehicles, can be readily realized under the same framework with only minor 
adjustments. 

Given the theoretical and practical importance of multi-robot (and more 
generally, multi-agent) routing in crowded settings, in future work, we would 
like to push robot density to be significantly higher than $50\%$. To 
achieve this while retaining optimality assurance, we believe the 
computation-based method developed in this work can be leveraged, perhaps in 
conjunction with a more sophisticated version of the \oldrdisc algorithm. On 
the other hand, a triangular grid supports a maximum density of $66\%$; it may 
be of interest to explore alternatives structures for accommodating 
denser robot configurations. 

\bibliographystyle{IEEEtran}
\bibliography{bib}
\end{document}